\documentclass[sigplan,screen,nonacm]{acmart}
\usepackage[utf8]{inputenc} 
\usepackage[T1]{fontenc}    
\usepackage{hyperref}       
\usepackage{url}            
\usepackage{booktabs}       
\usepackage{amsfonts}       
\usepackage{nicefrac}       
\usepackage{microtype}      
\usepackage{lipsum}
\usepackage{fancyhdr}       
\usepackage{graphicx}       
\usepackage{amsmath,amsthm}
\usepackage{tikz-cd}
\usepackage{tikz}
\usepackage{csquotes}
\usepackage{tabularx}
\usepackage{subcaption}
\usepackage{multirow}
\usepackage{pifont}
\usepackage{bm}
\usepackage{array}
\usepackage{hyperref}

\graphicspath{{media/}}     
\newtheorem{theorem}{Theorem}

\newtheorem{definition}{Definition}

\makeatletter
\def\@ACM@checkaffil{
	\if@ACM@instpresent\else
	\ClassWarningNoLine{\@classname}{No institution present for an affiliation}%
	\fi
	\if@ACM@citypresent\else
	\ClassWarningNoLine{\@classname}{No city present for an affiliation}%
	\fi
	\if@ACM@countrypresent\else
	\ClassWarningNoLine{\@classname}{No country present for an affiliation}%
	\fi
}
\makeatother

\title{GraphRNN Revisited: An Ablation Study and Extensions for Directed Acyclic Graphs}

\author{Taniya Das}
\affiliation{University of Oxford}
\authornote{All authors contributed equally to this work. Listing order is random.}
\author{Mark Koch}
\affiliation{University of Oxford}
\authornotemark[1]
\author{Maya Ravichandran}
\affiliation{University of Cambridge}
\authornotemark[1]
\authornote{mr894@cam.ac.uk}
\author{Nikhil Khatri}
\affiliation{University of Oxford}
\authornotemark[1]

\setcopyright{none}
\begin{abstract} GraphRNN is a deep learning-based architecture proposed by You et al. \cite{you2018graphrnn} for learning generative models for graphs.
We replicate the results of You et al. using a reproduced implementation of the GraphRNN architecture and evaluate this against baseline models using new metrics. Through an ablation study, we find that the BFS traversal suggested by You et al. to collapse representations of isomorphic graphs contributes significantly to model performance. Additionally, we extend GraphRNN to generate directed acyclic graphs by replacing the BFS traversal with a topological sort. We demonstrate that this method improves significantly over a directed-multiclass variant of GraphRNN on a real-world dataset.
\end{abstract} 
\begin{document}

\maketitle

\section{Introduction}
The paper \textit{GraphRNN: Generating Realistic Graphs with Deep Auto-regressive Models} by You et al. \cite{you2018graphrnn} proposes a deep auto-regressive model for graph generation.
They present the GraphRNN (Graph Recurrent Neural Network) architecture, which can generate variable-sized graphs using a hierarchy of two RNNs.
Their approach scales to graphs with large node counts by employing a breadth first search (BFS) ordering scheme that reduces the number of representations a given graph can have.
You et al. evaluate their method on a suite of synthetic and real-world graph datasets, introducing new metrics for quantitative evaluation based on the Maximum Mean Discrepancy (MMD) of various graph statistics.

We demonstrate that we have a performant implementation of the original GraphRNN model by reproducing the qualitative results and reported MMD scores for the datasets used by You et al. We focus on reproducing results on the smaller versions of the datasets due to limited compute resources. 
The code for our reproduced implementation, evaluation, and extensions is publicly available at
\begin{center}
    \texttt{\url{https://github.com/mark-koch/graph-rnn}}
\end{center}

\paragraph{Overview}
We begin with a brief summary of the GraphRNN architecture in Section \ref{sec:summary}.
Section \ref{sec:results} features the results of the reproduced experiments, followed by a discussion of alternative evaluation metrics in Section \ref{sec:alt_metrics}.
We run an ablation study in Section \ref{sec:bfs-ablation} to empirically investigate the importance of the BFS ordering for model performance.
Further, we present a novel extension for generating directed acyclic graphs (DAGs) and report improvements over an implementation of GraphRNN for multiclass directed graphs in Section \ref{sec:directed_graphs}.
Finally, we conclude by discussing our overall results in Section \ref{sec:discussion}.\footnote{This work was completed as part of a class assignment.}

\section{GraphRNN Summary}\label{sec:summary}

We briefly summarise the parts of the GraphRNN paper relevant to our experiments.

\subsection{Representing Graphs as Sequences}

\begin{figure}
    \begin{minipage}[c]{0.4\linewidth}
        \begin{tikzpicture}[every node/.style={shape=circle,draw=black,inner sep=0.75mm,font={\footnotesize\boldmath}}]
            \node (A) at (0,0) {$A$};
            \node (B) at (2,0) {$B$};
            \node (C) at (0,-2) {$C$};
            \node (D) at (2,-2) {$D$};
            \node (E) at (1,-1) {$E$};
            
            \path (A) edge (B);
            \path (A) edge (C);
            \path (C) edge (D);
            \path (B) edge (D);
            \path (A) edge (E);
            \path (C) edge (E);
        \end{tikzpicture}
    \end{minipage}%
    \begin{minipage}[c]{0.5\linewidth}
        \centering
        \newcounter{nodecount}
        \newcommand\tabnodeb[1]{\addtocounter{nodecount}{1} \tikz \node (\arabic{nodecount}) {#1};}
        \newcommand\tabnodeg[1]{\addtocounter{nodecount}{1} \tikz \node (\arabic{nodecount}) {\color{black!30} #1};}
        \tikzstyle{every picture}+=[remember picture,baseline]
        \setlength{\tabcolsep}{2pt}
        \vspace{0.5cm}
        \begin{tabular}{c|ccccc}
            \tabnodeb{A} & \tabnodeg{0} & \tabnodeb{1} & \tabnodeb{1} & \tabnodeb{0} & \tabnodeb{1} \\
            \tabnodeb{B} & \tabnodeg{1} & \tabnodeg{0} & \tabnodeb{0} & \tabnodeb{1} & \tabnodeb{0} \\
            \tabnodeb{C} & \tabnodeg{1} & \tabnodeg{0} & \tabnodeg{0} & \tabnodeb{1} & \tabnodeb{1} \\
            \tabnodeb{D} & \tabnodeg{0} & \tabnodeg{1} & \tabnodeg{1} & \tabnodeg{0} & \tabnodeb{0} \\
            \tabnodeb{E} & \tabnodeg{1} & \tabnodeg{0} & \tabnodeg{1} & \tabnodeg{0} & \tabnodeg{0} \\ \hline
                         & \tabnodeb{A} & \tabnodeb{B} & \tabnodeb{C} & \tabnodeb{D} & \tabnodeb{E} \\
        \end{tabular}
        \begin{tikzpicture}[overlay]
            \draw[black](3.north west) -- (3.north east) -- (3.south east) -- (3.south west) -- cycle;
            \draw[black](4.north west) -- (4.north east) -- (10.south east) -- (10.south west) -- cycle;
            \draw[black](5.north west) -- (5.north east) -- (17.south east) -- (17.south west) -- cycle;
            \draw[black](6.north west) -- (6.north east) -- (24.south east) -- (24.south west) -- cycle;
            
            \node [above=0.25cm] at (3) (A) {$S^\pi_2$};
            \node [above=0.25cm] at (4) (B) {$S^\pi_3$};
            \node [above=0.25cm] at (5) (C) {$S^\pi_4$};
            \node [above=0.25cm] at (6) (D) {$S^\pi_5$};
        \end{tikzpicture}
    \end{minipage}
    \caption{A graph and its adjacency matrix under the node ordering $A,B,C,D,E$. The highlighted sequence of upper triangular columns $S^\pi = (S_2^\pi, S_3^\pi, S_4^\pi, S_5^\pi)$ is enough to fully characterise the graph.}
    \label{fig:adjacency triangle}
\end{figure}
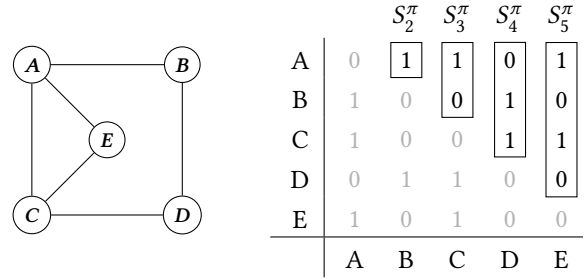

You et al. consider undirected graphs $G = (V, E)$ given by a set of nodes $V = \{ v_1,...,v_n \}$ and a symmetric edge relation $E \subseteq V\times V$.
Further, $E$ is assumed to be irreflexive, i.e. nodes do not have self-loops.

A typical representation of graphs are \textit{adjacency matrices}.
You et al. consider adjacency matrices for all possible permutations of nodes.
Given a permutation function ${\pi : V \to V}$, they define the adjacency matrix $A^\pi \in \{0,1\}^{n \times n}$ as
$$ A^\pi_{i,j} = \begin{cases}
    1 & \text{if } (\pi(v_i), \pi(v_j)) \in E \\
    0 & \text{otherwise.}
\end{cases} $$
The autoregressive model presented by You et al. generates graphs by generating their adjacency matrices column by column.
More precisely, it suffices to generate the upper triangle of $A^\pi$ since $A^\pi$ is symmetric for undirected graphs.
This is illustrated in Figure \ref{fig:adjacency triangle}.
Formally, undirected graphs are represented by a sequence
$$ S^\pi = (S^\pi_2,...,S^\pi_n) $$
where $S^\pi_i \in \{0,1\}^{i-1}$ is the $i$-th column of the upper right triangle of $A^\pi$, i.e.
$$ S^\pi_i = (A^\pi_{1,i}, ..., A^\pi_{i-1,i})^T. $$

Instead of learning to generate graph sequences $S^\pi$ for all possible permutations $\pi$, You et al. only consider node orderings that arise from breadth first search (BFS) traversals of the graph.
They claim that this approach has two main benefits:
\begin{itemize}
    \item 
    Since many graphs have less than $n!$ BFS orderings, this technique usually reduces the number of different representations a graph can have.
    
    \smallskip
    \item
    If $\pi$ is a BFS ordering, then for all $i < j$, $A^\pi_{i,j-1} = 1$ and $A^\pi_{i,j}=0$ implies that $A^\pi_{i',j'} = 0$ for all $i' \leq i$ and $j' \geq j$ (see Proposition 1 in \cite{you2018graphrnn}).
    As a consequence, the entries of $A^\pi$ are spread around the diagonal of the matrix as shown in Figure \ref{fig:adj_mat}.
    This allows us to shrink the size of the vector $S^\pi_i$ needed to represent each column to some fixed length $M \leq n$ that only includes the entries around the diagonal as illustrated in Figure \ref{subfig:adj_mat_m}.
    Formally,
    $$ S^\pi_i = (A^\pi_{\max(1,i-M)}, ..., A^\pi_{i-1,i})^T. $$
    This reduces the number of predictions the generation model needs to make.
   
\end{itemize}
You et al. highlight this BFS technique as a crucial component of their approach and claim that it drastically improves scalability.
We empirically investigate the contribution of BFS in Section \ref{sec:bfs-ablation}.

\begin{figure}
    \centering
    \begin{subfigure}{0.47\linewidth}
        \vspace{1.2pt}
         $n ~ \left\{ ~ \vcenter{\hbox{\includegraphics[width=0.8\linewidth]{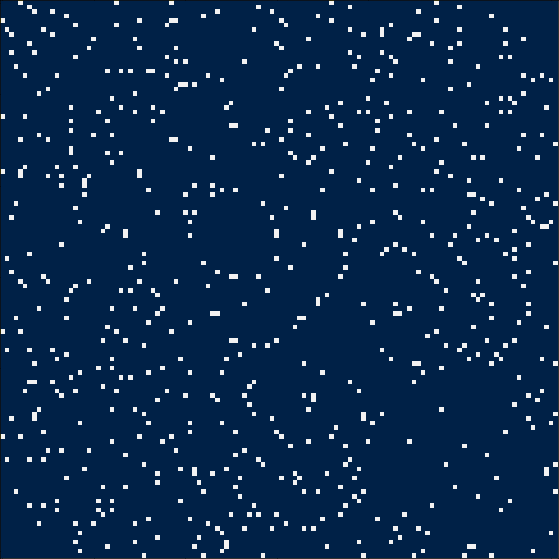}}} \right. \hphantom{nnn}$
         \vspace{1.2pt}
        \caption{\label{subfig:adj_mat_default}}
    \end{subfigure}
    \quad
    \begin{subfigure}{0.47\linewidth}
        $\hphantom{M}~ \vcenter{\hbox{\includegraphics[width=0.8\linewidth]{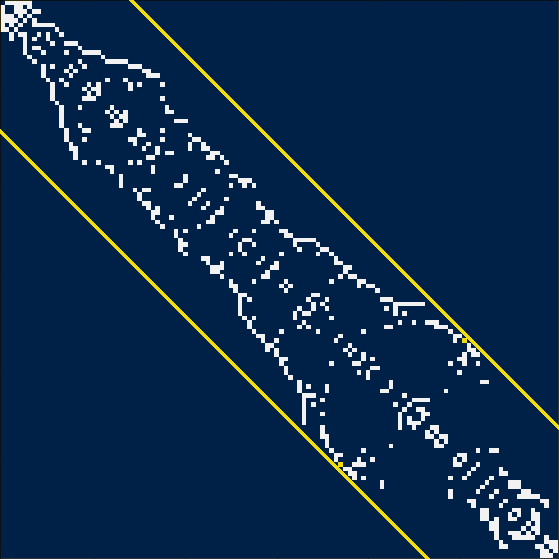}}} ~ \begin{aligned}\\[1.8cm] \left.\begin{aligned}\\[10pt]\end{aligned}\right\} M \end{aligned}$
        \caption{\label{subfig:adj_mat_m}}
    \end{subfigure}
    \caption{\label{fig:adj_mat}(a) Shows the adjacency matrix of a random graph under a random node ordering. (b) Depicts the adjacency matrix of the same graph under a BFS ordering. We can see that the entries are closer to the diagonal under the BFS ordering, allowing us to represent the upper triangle using smaller slices of length $M$.}
\end{figure}

\subsection{GraphRNN Architecture}

You et al. propose a hierarchy of two recurrent neural networks (RNNs) to autoregressively generate $S^\pi$.
The first RNN operates on the node level.
Given a sequence $S^\pi_1,...,S^\pi_i$ of previously generated adjacency vectors, it maintains a hidden representation $h_i$ that captures the global state of the graph.
This hidden state is used to initialise the second RNN that operates on the edge level and is used to autoregressively sample the components of $S^\pi_{i+1}$, i.e. the connectivity of the $i+1$ 'th node.
You et al. also introduce a simplified version of their model called GraphRNN-S in which the second edge-level RNN is replaced with a multi-layer perceptron (MLP) generating a vector $\theta_{i+1}$ of probabilities from which $S^\pi_{i+1}$ is sampled independently.

\begin{figure*}[t]
    \setlength{\tabcolsep}{0pt}
    \newcommand{\graphwidth}{2cm}
    \newcommand{\catsep}{5pt}
    \centerline{%
    \begin{tabular}{ccccc|ccccc|cccc}
        & \multicolumn{3}{c}{\textbf{Grid-small}} & & & \multicolumn{3}{c}{\textbf{Community-small}} & & & \multicolumn{3}{c}{\textbf{Ego-small}}
        \\[5pt]
        \rotatebox{90}{\hspace{5pt}\textbf{Train}} \hspace{\catsep} &
        \includegraphics[width=\graphwidth]{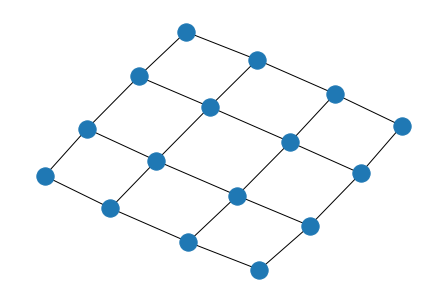} &
        \includegraphics[width=\graphwidth]{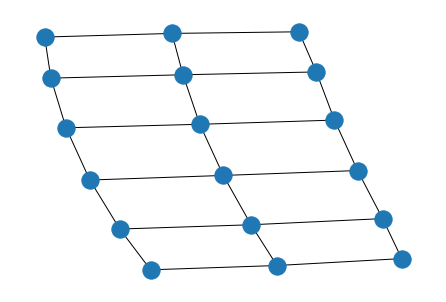} &
        \includegraphics[width=\graphwidth]{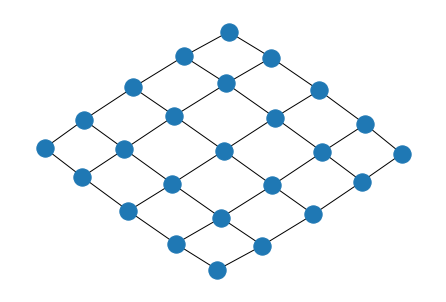} 
        & \hspace{\catsep} & \hspace{\catsep} &
        \includegraphics[width=\graphwidth]{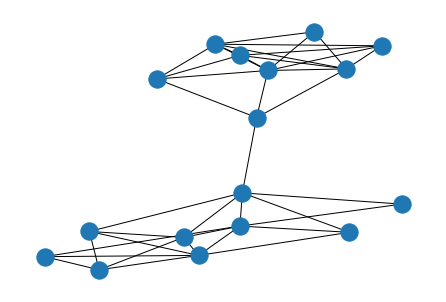} &
        \includegraphics[width=\graphwidth]{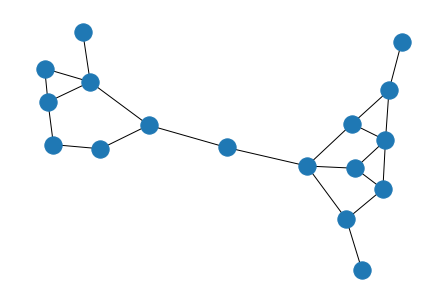} &
        \includegraphics[width=\graphwidth]{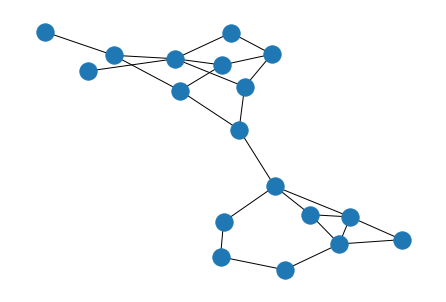} 
        & \hspace{\catsep} & \hspace{\catsep} &
        \includegraphics[width=\graphwidth]{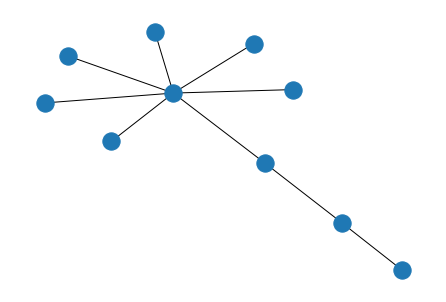} &
        \includegraphics[width=\graphwidth]{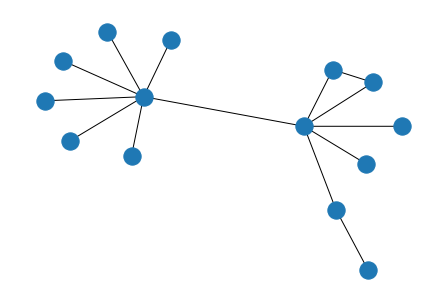} &
        \includegraphics[width=\graphwidth]{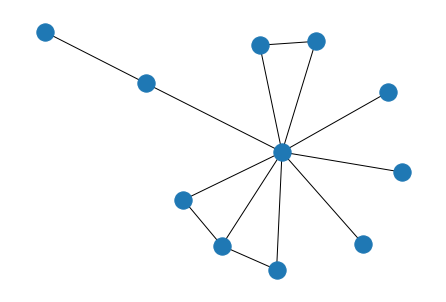}
        \\[5pt]
        \rotatebox{90}{\hspace{10pt}\textbf{Ours}} \hspace{\catsep} &
        \includegraphics[width=\graphwidth]{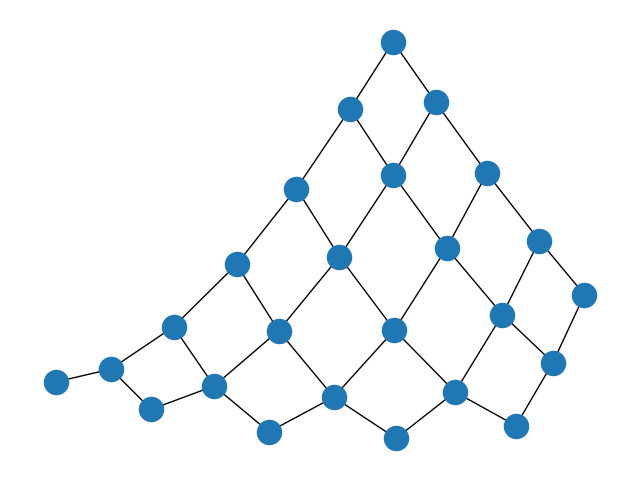} &
        \includegraphics[width=\graphwidth]{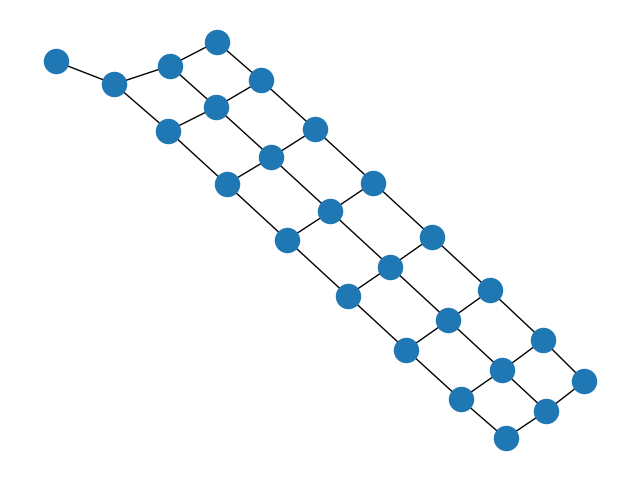} &
        \includegraphics[width=\graphwidth]{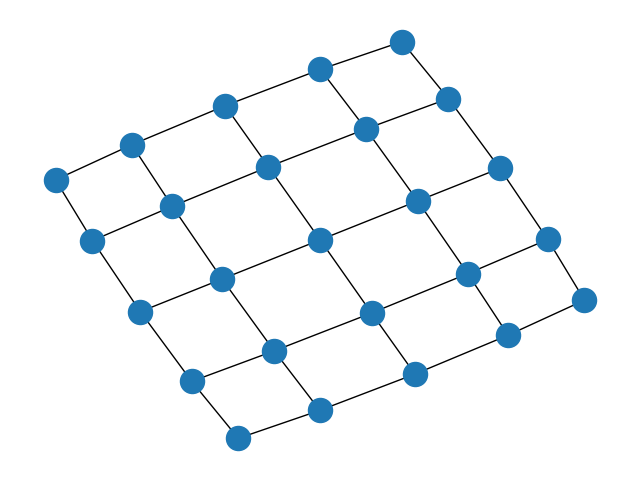}
        & \hspace{\catsep} & \hspace{\catsep} &
        \includegraphics[width=\graphwidth]{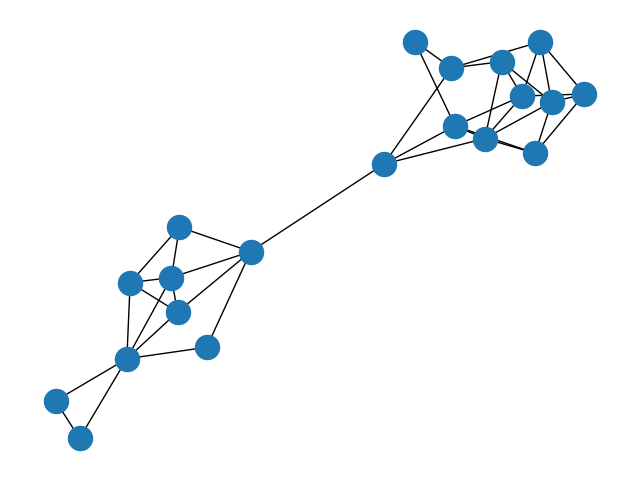} &
        \includegraphics[width=\graphwidth]{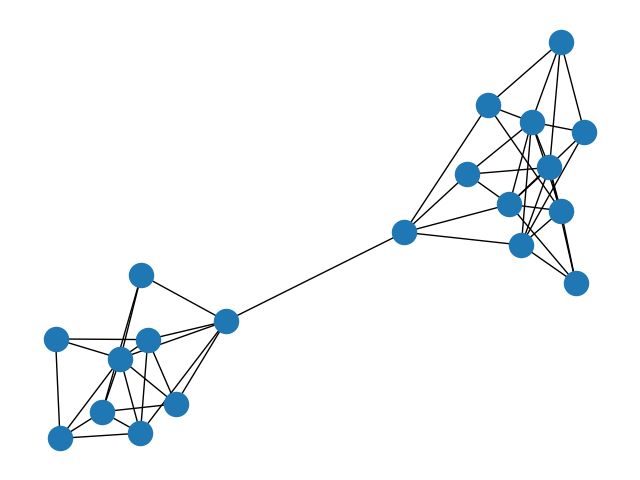} &
        \includegraphics[width=\graphwidth]{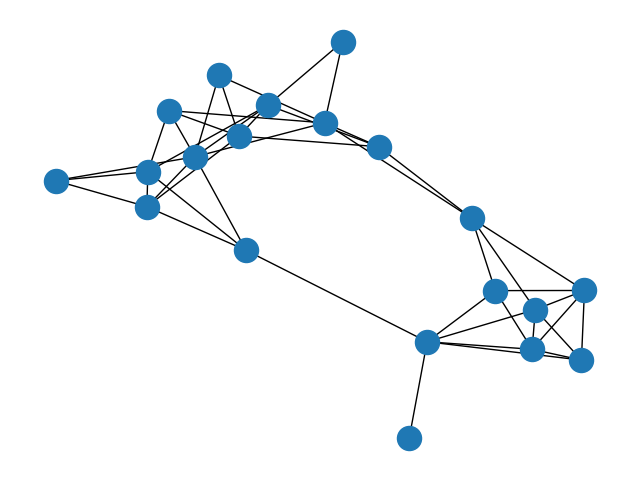}
        & \hspace{\catsep} & \hspace{\catsep} &
        \includegraphics[width=\graphwidth]{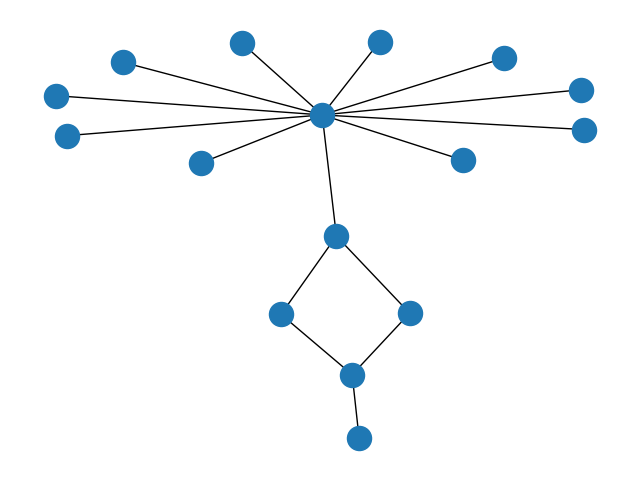} &
        \includegraphics[width=\graphwidth]{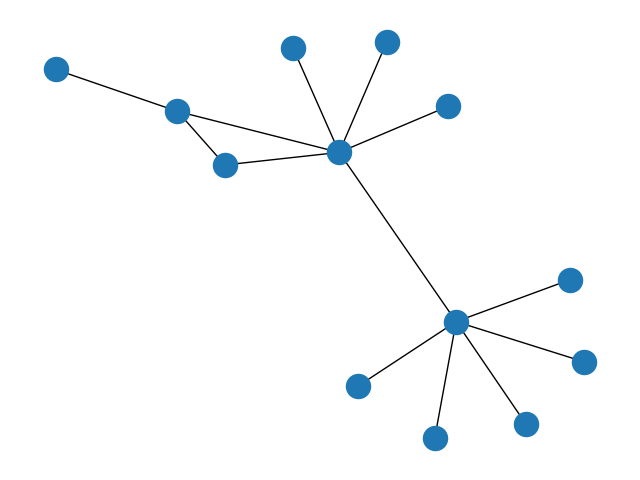} &
        \includegraphics[width=\graphwidth]{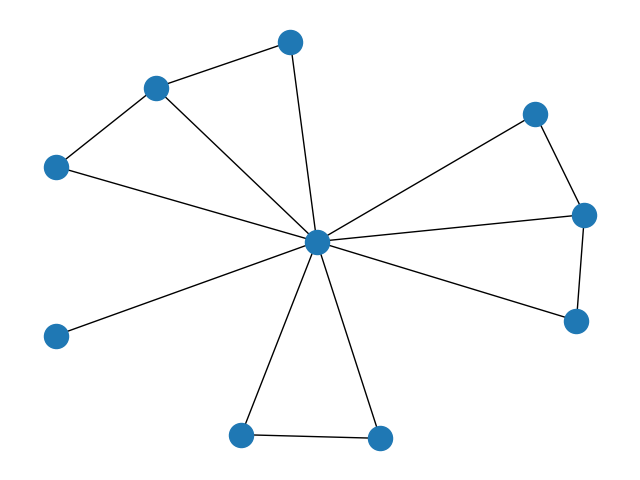}
    \end{tabular}}
    \caption{\label{fig:graphs} Generated graphs from our implementation of GraphRNN compared to the training data.}
\end{figure*}

\begin{figure*}[t]
    \centering
    \begin{subfigure}{0.45\linewidth}
        \centering
        \includegraphics[width=0.8\linewidth]{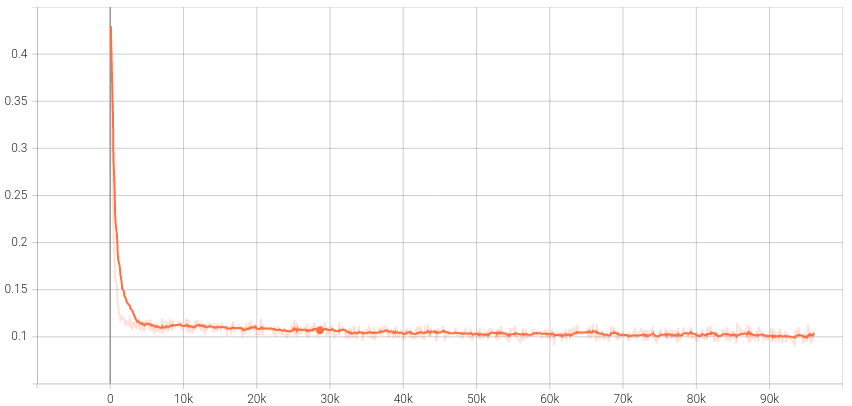}
        \caption{}
    \end{subfigure}
    \begin{subfigure}{0.45\linewidth}
        \centering
        \includegraphics[width=0.8\linewidth]{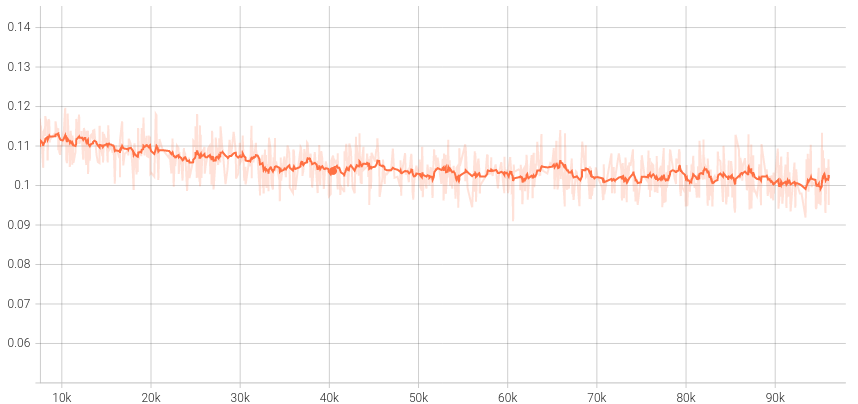}
        \caption{}
    \end{subfigure}
    \caption{\label{fig:training-curves} Loss curve for training GraphRNN on the \textit{community-small} datset. Part (a) shows the full curve and part (b) the curve after 10k iterations.}
\end{figure*}

\subsection{Evaluation Metrics}

Quantitatively evaluating the quality of a generated graph is a challenging task. Graph-matching is in general intractable, and thus You et al. suggest quantifying the quality of generated graphs through comparison of a set of graph statistics.
To compare node-level characteristics, such as degree and clustering coefficient between two graphs, they suggest the use of maximum mean discrepancy (MMD) to measure the similarity of their distributions. 

You et al. use the MMD measure on three metrics: degree, clustering coefficients, and orbit statistics \cite{Graphlet_counting} distributions.

\subsection{Experiments}

\paragraph{Datasets}
You et al. evaluate their models on two synthetic and two real-world graph datasets.
The synthetic ones include a \textit{community} dataset (two-community graphs with $60 \leq |V| \leq 160$ generated using the Erd\H{o}s-R\'enyi model (ER)~\cite{erdHos1960evolution}) and a \textit{grid} dataset (2D rectangular grid graphs with side length ranging from 10 to 20).
The Citeseer network \cite{sen2008collective} is used to generate 3-hop graphs with $50 \leq |V| \leq 399$, forming the \textit{ego} dataset.
Finally, You et al. use a \textit{protein} dataset consisting of graphs with $100 \leq |V| \leq 500$ representing protein structures. In addition, they evaluate smaller equivalents of these datasets.

\paragraph{Baselines}
You et al. compare their methods against the rule-based Erd\H{o}s-R\'enyi (ER) \cite{erdHos1960evolution} and Barab\'asi-Albert (BA)~\cite{albert2002statistical} models that generate graphs through random processes. They also compare their methods against Kronecker graph models~\cite{leskovec2010kronecker}, mixed-membership stochastic block models \cite{airoldi2008mixed}, and the deep learning approaches GraphVAE \cite{simonovsky2018graphvae} and DeepGMG~\cite{li2018learning}.

\section{Reproduced Experiments}\label{sec:results}

\begin{table*}[t]
\caption{\label{tab:results_comprehensive} Comparison of our GraphRNN implementation with the original implementation and baselines, using the MMD measures of degree, clustering and orbit distributions}
\begin{tabular}{llrrrrrrrrr}
 & & \multicolumn{3}{c}{\textbf{community-small}} & \multicolumn{3}{c}{\textbf{ego-small}} & \multicolumn{3}{c}{\textbf{grid-small}} \\ \cline{3-11} 
 & & \multicolumn{1}{l}{\textbf{deg.}} & \multicolumn{1}{l}{\textbf{clus.}} & \multicolumn{1}{l|}{\textbf{orbit}} & \multicolumn{1}{l}{\textbf{deg.}} & \multicolumn{1}{l}{\textbf{clus.}} & \multicolumn{1}{l|}{\textbf{orbit}} & \multicolumn{1}{l}{\textbf{deg.}} & \multicolumn{1}{l}{\textbf{clus.}} & \multicolumn{1}{l}{\textbf{orbit}} \\ \cline{3-11} 
\multicolumn{1}{l|}{\multirow{2}{*}{\textbf{Baselines}}} & \multicolumn{1}{l|}{\textbf{BA}} & 0.459 & 0.707 & \multicolumn{1}{r|}{0.083} & 0.159 & 0.589 & \multicolumn{1}{r|}{0.012} & 1.206 & \textbf{0} & 0.084 \\
\multicolumn{1}{l|}{} & \multicolumn{1}{l|}{\textbf{ER}} & \textbf{0.016} & 0.193 & \multicolumn{1}{r|}{0.012} & 0.136 & 0.050 & \multicolumn{1}{r|}{0.035} & 0.242 & 0.656 & 0.024 \\ \hline
\multicolumn{1}{l|}{\multirow{2}{*}{\textbf{You et al. \cite{you2018graphrnn}}}} & \multicolumn{1}{l|}{\textbf{GraphRNN-S}} & 0.020 & 0.150 & \multicolumn{1}{r|}{0.010} & 0.002 & 0.050 & \multicolumn{1}{r|}{\textbf{0.001}} & \multicolumn{1}{l}{NA} & \multicolumn{1}{l}{NA} & \multicolumn{1}{l}{NA} \\
\multicolumn{1}{l|}{} & \multicolumn{1}{l|}{\textbf{GraphRNN}} & 0.030 & 0.030 & \multicolumn{1}{r|}{0.010} & \bm{$3\times10^{-4}$} & 0.050 & \multicolumn{1}{r|}{\textbf{0.001}} & \multicolumn{1}{l}{NA} & \multicolumn{1}{l}{NA} & \multicolumn{1}{l}{NA} \\ \hline
\multicolumn{1}{l|}{\multirow{2}{*}{\textbf{Ours}}} & \multicolumn{1}{l|}{\textbf{GraphRNN-S}} & 0.038 & \textbf{0.007} & \multicolumn{1}{r|}{0.006} & 0.004 & \textbf{0.032} & \multicolumn{1}{r|}{\textbf{0.001}} & 0.080 & 0.065 & 0.007 \\
\multicolumn{1}{l|}{} & \multicolumn{1}{l|}{\textbf{GraphRNN}} & 0.031 & 0.012 & \multicolumn{1}{r|}{\textbf{0.001}} & 0.004 & 0.038 & \multicolumn{1}{r|}{0.003} & \textbf{0.065} & \textbf{0} & \textbf{0.001} \\\hline 
\end{tabular}
\end{table*}

For our reproducibility experiments, we focus on the scaled-down datasets \textit{community-small} and \textit{ego-small}. 
Further, we generated our own \textit{grid-small} dataset containing smaller grids of size up to $6\times 6$.
Since grid-small is not part of the original paper, we applied the BA and ER baselines to show a similar performance improvement for this new dataset.

We implemented the GraphRNN architecture in PyTorch following the instructions in \cite{you2018graphrnn}.
In particular, we use the configuration for the smaller model mentioned in Appendix A.1 to match the small datasets and to speed up training. In dataset generation, when there were discrepancies between You et al.'s paper and code, we used the parameters mentioned in the paper. 
For the MMD measure and orbit metric, we used You et al.'s implementation. 

We use a random 80\%/20\% train/test split and train the models for 96k iterations with a batch size of 32, as mentioned in \cite{you2018graphrnn}.
However, we notice that training usually converges after 10k-30k iterations, which is likely due to the smaller size of the datasets and the smaller GraphRNN configuration used.
Across our experiments, we thus use checkpoints from the first 30k iterations of training.
Figure~\ref{fig:training-curves} shows a typical loss curve observed during training. Note that we
applied an accelerated learning rate decay schedule compared to \cite{you2018graphrnn} due to using smaller model configurations.

\begin{table*}[t]
\caption{\label{tab:results_addnl} Comparison between our implementation of GraphRNN and the BA and ER baselines. We use the MMD of betweeness centrality (BC) and closeness centrality (CC), as well as the absolute difference of density (D) and transitivity (T).}
\begin{tabular}{lrrrrrrrrrrrr}
 & \multicolumn{4}{c}{\textbf{community-small}} & \multicolumn{4}{c}{\textbf{ego-small}} & \multicolumn{4}{c}{\textbf{grid-small}} \\ \cline{2-13} 
 & \multicolumn{1}{l}{\textbf{BC}} & \multicolumn{1}{l}{\textbf{CC}} & \multicolumn{1}{l}{\textbf{D}} & \multicolumn{1}{l|}{\textbf{T}} & \multicolumn{1}{l}{\textbf{BC}} & \multicolumn{1}{l}{\textbf{CC}} & \multicolumn{1}{l}{\textbf{D}} & \multicolumn{1}{l|}{\textbf{T}} & \multicolumn{1}{l}{\textbf{BC}} & \multicolumn{1}{l}{\textbf{CC}} & \multicolumn{1}{l}{\textbf{D}} & \multicolumn{1}{l}{\textbf{T}} \\ \hline
\multicolumn{1}{l|}{BA} & 0.054 & 0.026 & 0.060 & \multicolumn{1}{r|}{0.297} & 0.092 & 0.058 & 0.049 & \multicolumn{1}{r|}{0.197} & 0.436 & 0.393 & 0.082 & \textbf{0} \\
\multicolumn{1}{l|}{ER} & 0.038 & 0.022 & 0.055 & \multicolumn{1}{r|}{0.205} & 0.065 & 0.049 & 0.075 & \multicolumn{1}{r|}{0.189} & \textbf{0.409} & \textbf{0.368} & 0.026 & 0.256 \\ \hline
\multicolumn{1}{l|}{GraphRNN-S} & \textbf{0.024} & \textbf{0.021} & 0.045 & \multicolumn{1}{r|}{0.167} & \textbf{0.062} & \textbf{0.045} & 0.068 & \multicolumn{1}{r|}{0.172} & 0.430 & 0.384 & 0.021 & 0.016 \\
\multicolumn{1}{l|}{GraphRNN} & \textbf{0.024} & 0.022 & \textbf{0.044} & \multicolumn{1}{r|}{\textbf{0.166}} & 0.085 & 0.054 & \textbf{0.039} & \multicolumn{1}{r|}{\textbf{0.146}} & 0.470 & 0.381 & \textbf{0.020} & \textbf{0} \\ \hline
\end{tabular}
\end{table*}

\subsection{Results}

Figure \ref{fig:graphs} visualises graphs generated using our implementation of GraphRNN compared to graphs from the training set.
This shows that our implementation is capable of generating the structure of grid, community, and ego graphs, matching the qualitative results presented in Figure 2 in the original paper.
Similar to You et al., we also observe that the model generates grids that are not present in the training data, for example, the $3\times 8$ grid shown in Figure \ref{fig:graphs}.

Table \ref{tab:results_comprehensive} presents a quantitative comparison of our implementation of GraphRNN with You et al.'s results and baseline results for the MMD scores. We generate baseline metrics using the BA and ER models, since the GraphRNN paper does not provide these baselines for the smaller datasets. 
Overall, we observe similar improvements over the baselines as reported in \cite{you2018graphrnn}.

\section{Alternate Graph Similarity Statistics}\label{sec:alt_metrics}

The graph similarity metrics considered by You et al. \cite{you2018graphrnn} (degree distribution, clustering distribution, and orbit statistics) are primarily concerned with individual node statistics and local structures in graphs. To provide further quantitative insight into the properties of graphs generated by GraphRNN, we consider some additional metrics, described below.

\paragraph{Betweenness Centrality (BC)}
Betweenness centrality for a node is the proportion of shortest paths in the graph that pass through that node.
Compared to degree, clustering, and orbit statistics that mainly focus on local graph structures, BC provides a view of the global connectivity of graphs.
Following the previous experiments, we make use of MMD to compare the BC distributions of different graphs.

\paragraph{Closeness Centrality (CC)}
Closeness centrality is a node-level measure which is defined as the reciprocal of the sum of shortest paths to all other nodes in the graph.
Similar to BC, CC captures global graph connectivity. Again, we compare distributions via MMD.

\paragraph{Density (D)}
Density is graph-level measure defined as the ratio between the number of edges present in the graph and the number of possible edges in the graph if each node in the graph were connected to all other nodes. We compare the mean absolute density difference between generated graphs and the graphs in the test set.

\paragraph{Transitivity (T)}
Transitivity is a graph-level property defined as
$$t = \frac{3 \times \text{\#triangles}}{\text{\#triads}},$$
where triads are 3 nodes sharing at least 2 edges. Like the clustering coefficient, transitivity is a measure of the graph property of triadic closure. 
Unlike the clustering coefficient that only describes how densely single nodes are connected, transitivity measures how transitive the edge relation $E$ is.
Similar to density, we evaluate the mean absolute difference of transitivity between generated graphs and the test set.

\begin{table*}[t]
\centering
\caption{\label{tab:mvalues}M values for considered datasets, together with the reduction they represent from the maximum number of nodes in a graph in each dataset}
\begin{tabular}{l|rrr|r}
\textbf{Dataset} & \multicolumn{1}{l}{\textbf{Min nodes}} & \multicolumn{1}{l}{\textbf{Max nodes}} & \multicolumn{1}{l}{\textbf{M}} & \multicolumn{1}{|l}{\textbf{Reduction}} \\ \hline
ego-small        & 4 & 18 & 16 & 11.11\% \\
community-small  & 12 & 20 & 16 & 20\% \\
grid-small       & 4 & 36 & 12 & 66.67\% \\\hline
\end{tabular}
\end{table*}

\subsection{Results}

Table \ref{tab:results_addnl} presents a comparison of our GraphRNN implementation with baseline models for the alternative statistics described above.
For the community-small dataset, we observe that both GraphRNN-S and GraphRNN outperform the baselines for all new metrics. We observe that there is negligible difference in the performance of GraphRNN-S and GraphRNN. This is consistent with the results of Table \ref{tab:results_comprehensive}, where both models perform similarly. For the ego-small dataset, the best score for each metric is achieved by the GraphRNN or GraphRNN-S models. For the grid-small dataset, we observe that the ER model achieves slightly lower values for the centrality measures.
The BA model is able to perfectly replicate the transitivity of the grid-small dataset. This is also consistent with BA achieving a perfect score for clustering coefficient (another measure of triadic closure) on the grid-small dataset in Table \ref{tab:results_comprehensive}. In both cases, our implementation of GraphRNN is able to match this result.

\enlargethispage{\baselineskip}

\begin{table*}[t]
\centering
\caption{\label{tab:results_ablation}Comparison of the GraphRNN model with the GraphRNN-no-M and GraphRNN-no-M ablation models using the MMD measures of degree, clustering and orbit distributions}
\begin{tabular}{lllrrrrrrrrr}
 & & & \multicolumn{3}{c}{\textbf{community-small}} & \multicolumn{3}{c}{\textbf{ego-small}} & \multicolumn{3}{c}{\textbf{grid-small}} \\ \cline{4-12} 
\textbf{BFS} & \textbf{M} & & \multicolumn{1}{l}{\textbf{deg.}} & \multicolumn{1}{l}{\textbf{clus.}} & \multicolumn{1}{l|}{\textbf{orbit}} & \multicolumn{1}{l}{\textbf{deg.}} & \multicolumn{1}{l}{\textbf{clus.}} & \multicolumn{1}{l|}{\textbf{orbit}} & \multicolumn{1}{l}{\textbf{deg.}} & \multicolumn{1}{l}{\textbf{clus.}} & \multicolumn{1}{l}{\textbf{orbit}} \\ \hline
\checkmark & \checkmark & \multicolumn{1}{l|}{\textbf{GraphRNN}} & \textbf{0.031} & \textbf{0.012} & \multicolumn{1}{r|}{\textbf{0.001}} & \textbf{0.004} & 0.038 & \multicolumn{1}{r|}{0.003} & \textbf{0.065} & \textbf{0.000} & \textbf{0.001} \\
\checkmark & $\times$ & \multicolumn{1}{l|}{\textbf{GraphRNN}} & \textbf{0.031} & 0.027 & \multicolumn{1}{r|}{\textbf{0.001}} & 0.008 & \textbf{0.023} & \multicolumn{1}{r|}{\textbf{0.001}} & 0.073 & 0.126 & 0.018 \\
$\times$ & $\times$ & \multicolumn{1}{l|}{\textbf{GraphRNN}} & 0.034 & 0.107 & \multicolumn{1}{r|}{0.023} & 0.067 & 0.096 & \multicolumn{1}{r|}{0.022} & 0.134 & 0.660 & 0.018 \\ \hline
\end{tabular}
\end{table*}

\section{Ablation - Investigating the Contribution of BFS} \label{sec:bfs-ablation}

We perform an ablation study to test You et al.'s claims regarding the benefit of BFS ordering of the graph adjacency matrix. The purpose of this study is to determine the extent to which the reduction in the number of graph representations and the reduced number of edges that the model has to predict improve the performance. The following experiments are performed:

\paragraph{GraphRNN-no-M} This experiment aims to assess the effect of the reduction in number of edges the model has to predict  on performance. Here, the nodes are arranged in BFS order as in the default GraphRNN model, but the input sequence is not truncated to be of size $M$. 

\paragraph{GraphRNN-no-BFS:} This experiment aims to assess the effect of the reduction in the number of graph representations on performance. Here, the adjacency matrix for each graph is permuted randomly, and the vector sequence $ S^\pi = (S^\pi_2,...,S^\pi_n) $ is fed into the model without further processing.\footnote{Note that since a BFS traversal is not used for $\pi$, an $M$ value cannot be calculated.} 

\smallskip

Both these experiments are run for the \textit{grid-small}, \textit{ego-small} and \textit{community-small} datasets.
Table \ref{tab:mvalues} contains the minimum and maximum number of nodes in each considered dataset, together with the $M$ values calculated by considering 9000 random BFS permutations of graphs from each. The grid-small dataset has a particularly low value of $M$ (which we attribute to the low density and node degrees of grid graphs), yielding a significant reduction in the number of edges the model needs to predict. In comparison, the ego-small and community-small datasets have $M$ values much closer to the maximum number of nodes.

We make the following hypotheses about the ablation study:

\paragraph{\textit{Hypothesis 1:}}In graphs where the $M$ value is close to the maximum number of nodes, not much performance is gained through the use of $M$-reduced data.
\paragraph{\textit{Hypothesis 2:}} BFS traversal-based ordering contributes significantly to performance in all cases.

\smallskip

\textit{Hypothesis 1} is suggested by the fact that when M is close to the number of nodes in a graph, it does not reduce the number of edge predictions required from the model by an appreciable amount. Thus, the model's performance should, at best, be only mildly improved.
\textit{Hypothesis 2} is motivated by the fact that there are typically far fewer than $n!$ BFS traversals of a graph (an argument also used by You et al. in favor of the BFS traversal). These hypotheses are investigated experimentally below.

\subsection{Ablation Results}
The results of our ablation experiments are shown in Table \ref{tab:results_ablation}.
For the community-small dataset, we observe that there is no decrease in performance on removing the $M$-reduction from the model for the degree and orbit metrics. There is a slight decrease in performance as measured along the clustering MMD metric. This is consistent with \textit{Hypothesis 1}, since the $M$ value for community-small is only 20\% less than the maximum number of nodes in a graph for this dataset, as shown in Table \ref{tab:mvalues}. On removing BFS from the model, a substantial decrease in performance is visible. The degree metric is slightly worse, while the clustering and orbit metrics are worse by an order of magnitude.

In the case of ego-small, on removing $M$, the model exhibits slightly better results for clustering and orbit, and slightly worse results for degree. Since the $M$ value reduction for this dataset is only 11\%, these results are also consistent with \textit{Hypothesis 1}. On removing BFS, we observe a consistent decrease in performance across all 3 metrics. Degree and orbit metrics worsen by an order of magnitude.

For the grid dataset, removing $M$ substantially worsens clustering and orbit MMD scores and mildly worsens the degree MMD score. Observe that for grid, the $M$ value reduction is significant, at 67\%. Thus, the observed sizeable regression in performance provides evidence to support \textit{Hypothesis 1}. Removing BFS from the model further worsens performance for degree and clustering metrics, while not impacting orbit.

\textit{Hypothesis 1} is thus experimentally supported, with the removal of $M$ significantly impacting only the datasets with high $M$-reduction percentages. Across datasets, we observe that removing BFS significantly hampers performance. This provides evidence in support of \textit{Hypothesis 2}.

\begin{table*}[t]
\centering
\caption{\label{tab:directed-results} Comparison of GraphRNN-DIR and GraphRNN-DAG models for the \textit{ego-directed} dataset, using MMD measures of degree and clustering}
\begin{tabular}{l|r|r|r}
     & \multicolumn{1}{l|}{\textbf{Degree}} & \multicolumn{1}{l|}{\textbf{Clustering}} & \multicolumn{1}{l}{\textbf{Components / graph}} \\ \hline
    \textbf{GraphRNN-DIR} & 0.746 & 1.343 & \textbf{1} \\
    \textbf{GraphRNN-DAG} & \textbf{0.054} & \textbf{0.125} & 1.3 \\ \hline
    \end{tabular}
\end{table*}

\section{Directed Graphs}
\label{sec:directed_graphs}
The original GraphRNN model considers only the case of undirected graphs. However, many graphs in practical applications are directed. In a directed graph, the edge relation $E$ need not be symmetric.
Under this relaxation of $E$, the adjacency matrix $A$ is no longer necessarily symmetric. Note that we continue to assume irreflexivity. We consider two avenues of extending GraphRNN to generate various classes of directed graphs.

\subsection{Generating Arbitrary Directed Graphs}

\renewcommand{\S}[2]{S^{\to #1}_{#2}}
\newcommand{\A}[2]{A^{\to #1}_{#2}}

So far, the graph sequences $S^\pi$ used to model graphs only contained zeros and ones, indicating whether nodes are connected or not.
In order to encode directed graphs, we introduce the notion of a directed graph sequence $\S{\pi}{}$ that also captures the direction of edges:
$$  \S{\pi}{} = (\S{\pi}{i}, ..., \S{\pi}{n}) $$
where
$$ \S{\pi}{i} = (\A{\pi}{1,i}, ..., \A{\pi}{i-1,i})^T \quad \text{for } 2 \leq i \leq n,$$
$$ \A{\pi}{i,j} = \begin{cases}
    (1,0,0,0) & \text{if } (\pi_i,\pi_j) \not\in E \land (\pi_j,\pi_i) \not\in E, \\
    (0,1,0,0) & \text{if } (\pi_i,\pi_j) \in E \land (\pi_j,\pi_i) \not\in E, \\
    (0,0,1,0) & \text{if } (\pi_i,\pi_j) \not\in E \land (\pi_j,\pi_i) \in E, \\
    (0,0,0,1) & \text{if } (\pi_i,\pi_j) \in E \land (\pi_j,\pi_i) \in E. \\
\end{cases} $$

You et al. mention directed graphs as a special instance of graphs with edge-features. Specifically, each edge belongs to one of two classes, indicating its direction.
We use a slight modification of this proposal. In our implementation, we encode the edge class as a one-hot vector of length 4. We use a softmax activation as the final layer of our edge-level model (both RNN and MLP) to generate a probability distribution over edge classes.
This extension of GraphRNN, which represents directed graphs through multiple edge classes, is called \textit{GraphRNN-DIR}. Nodes are ordered by a BFS traversal of the graph which respects edge directions.

\subsection{Generating Directed Acyclic Graphs}
While the above modification to GraphRNN is suitable for the general case of directed graphs, several interesting graphs exist with known structural constraints. For example, several graphs of interest are known to be directed acyclic graphs (DAGs).
We present a novel extension of GraphRNN to the special case of generating DAGs.

Recall that the input sequence to the GraphRNN model $S^{\pi}$ consists of the upper right triangular component of the graph's adjacency matrix.
We can formalise the notion of upper triangular matrices as follows:
\begin{definition}
    A matrix \,$A$ is upper right triangular if $A_{i, j} = 0$ for all $i,j$ with $i \geq j$.
\end{definition}
It turns out there is an interesting connection between DAGs and triangular matrices:
\begin{theorem} \label{thm:adj-topsort}
    An adjacency matrix $A^\pi$ is upper right triangular iff $\pi$ is a topological ordering of the associated graph.
\end{theorem}
\begin{proof}
    We prove both directions separately:
    
    \smallskip
    $\Rightarrow$:
    Suppose $A^\pi$ is an upper triangular adjacency matrix of some graph $G$ under some ordering $\pi$.
    Let $j < i$.
    We have to show that there is no edge from $\pi(v_i)$ to $\pi(v_j)$.
    This trivially holds since $A^\pi_{i, j} = 0$ as $A^\pi$ is upper triangular.
    Thus, $\pi$ is a topological sort.
        
    \smallskip
    $\Leftarrow$:
    Suppose the ordering $\pi$ represents a topological sort.
    We claim that $A^\pi$ is an upper triangular matrix.
    Let $i \geq j$.
    Then, $\pi(v_i)$ is the $i$-th node in the topological sort, and $\pi(v_j)$ is the $j$-th node in the topological sort.
    Since $i \geq j$, there cannot be an edge from the $i$-th node to the $j$-th node (otherwise the $i$-th node would need to appear before the $j$-th node in $\pi$, giving rise to a contradiction).
    Therefore, we must have $A^\pi_{i,j} = 0$ such that $A^\pi$ is upper triangular.
    
\end{proof}
Thus, we can describe DAGs using the original sequence representation $S^\pi$ if $\pi$ is a topological sort.
We propose to replace the previously used BFS ordering with a topological ordering of the nodes and then apply the original GraphRNN model to learn the distribution of those sequences.
Interpreting the generated sequences as an upper triangular matrix ensures that all graphs generated by the model are DAGs. We call this architecture GraphRNN-DAG.

A downside of this method is that we are no longer able to prune the dimension of the adjacency vectors $S^\pi_i$ to a smaller size $M$, since the topological sort does not reduce the number of necessary edge-predictions in the manner BFS does.

\subsection{Evaluation}
To evaluate our topological sort approach's performance, we consider a modification of the ego-small dataset. This is a natural choice since citation graphs are largely acyclic and have a direction associated with each edge (indicating either a \enquote{cites} or a \enquote{cited-by} relation). We take the largest weakly-connected component of the citeseer dataset \cite{giles1998citeseer}, and generate 200 directed 3-hop\footnote{The hops ignore edge direction.} ego networks centering on random nodes in this graph. This dataset has $7 \leq |V| \leq 30$, and an $M$ value of 25. This data is then passed to 2 models: \textit{GraphRNN-DIR} and \textit{GraphRNN-DAG}, described in the preceding sections. The graphs generated from these trained models were evaluated using MMD comparison of degree and clustering metrics with the test set. We omit the comparison of orbit MMD since the implementation is limited to undirected graphs.

As shown in Table \ref{tab:directed-results}, the GraphRNN-DAG model outperforms the multiclass approach (GraphRNN-DIR) along each MMD metric by a significant margin. One weakness of the topological sort model is that it occasionally generates graphs with multiple weakly connected components. This does not mirror the training dataset, where each graph is a single weakly connected component.

\section{Discussion}\label{sec:discussion}
In this work, we implemented the GraphRNN model proposed by You et al. \cite{you2018graphrnn} and reproduced their numerical results for small datasets. We evaluated this implementation against BA and ER baselines using new metrics, finding improvements across the board.
The results of our ablation study in Section \ref{sec:bfs-ablation} show that the BFS traversal contributes significantly to model performance across datasets.
We extended GraphRNN to directed graphs by introducing multiclass predictions. We further presented a novel extension of GraphRNN to directed acyclic graphs, and evaluated this against the GraphRNN-DIR implementation. A key insight from our work is that using a topological sort in place of a BFS traversal yields improved performance for generating DAGs. An avenue for future research is to identify other domain-aware traversals which may yield improved isomorphism-reductions for specific datasets.

\bibliographystyle{unsrt}  
\bibliography{references}

\begin{thebibliography}{10}

\bibitem{you2018graphrnn}
Jiaxuan You, Rex Ying, Xiang Ren, William Hamilton, and Jure Leskovec.
\newblock Graph{RNN}: Generating realistic graphs with deep auto-regressive
  models.
\newblock In {\em International conference on machine learning}, pages
  5708--5717. PMLR, 2018.

\bibitem{Graphlet_counting}
Tomaž Hočevar and Janez Demšar.
\newblock A combinatorial approach to graphlet counting.
\newblock {\em Bioinformatics}, 30:559--565, 2014.

\bibitem{erdHos1960evolution}
Paul Erd{\H{o}}s, Alfr{\'e}d R{\'e}nyi, et~al.
\newblock On the evolution of random graphs.
\newblock {\em Publ. Math. Inst. Hung. Acad. Sci}, 5(1):17--60, 1960.

\bibitem{sen2008collective}
Prithviraj Sen, Galileo Namata, Mustafa Bilgic, Lise Getoor, Brian Galligher,
  and Tina Eliassi-Rad.
\newblock Collective classification in network data.
\newblock {\em AI magazine}, 29(3):93--93, 2008.

\bibitem{albert2002statistical}
R{\'e}ka Albert and Albert-L{\'a}szl{\'o} Barab{\'a}si.
\newblock Statistical mechanics of complex networks.
\newblock {\em Reviews of modern physics}, 74(1):47, 2002.

\bibitem{leskovec2010kronecker}
Jure Leskovec, Deepayan Chakrabarti, Jon Kleinberg, Christos Faloutsos, and
  Zoubin Ghahramani.
\newblock Kronecker graphs: an approach to modeling networks.
\newblock {\em Journal of Machine Learning Research}, 11(2), 2010.

\bibitem{airoldi2008mixed}
Edo~M Airoldi, David Blei, Stephen Fienberg, and Eric Xing.
\newblock Mixed membership stochastic blockmodels.
\newblock {\em Advances in neural information processing systems}, 21, 2008.

\bibitem{simonovsky2018graphvae}
Martin Simonovsky and Nikos Komodakis.
\newblock Graph{VAE}: Towards generation of small graphs using variational
  autoencoders.
\newblock In {\em International conference on artificial neural networks},
  pages 412--422. Springer, 2018.

\bibitem{li2018learning}
Yujia Li, Oriol Vinyals, Chris Dyer, Razvan Pascanu, and Peter Battaglia.
\newblock Learning deep generative models of graphs.
\newblock {\em arXiv preprint arXiv:1803.03324}, 2018.

\bibitem{giles1998citeseer}
C~Lee Giles, Kurt~D Bollacker, and Steve Lawrence.
\newblock Citeseer: An automatic citation indexing system.
\newblock In {\em Proceedings of the third ACM conference on Digital
  libraries}, pages 89--98, 1998.

\end{thebibliography}

\end{document}